\documentclass[11pt]{article}

\usepackage{fullpage,times,url}

\usepackage{amsthm,amsfonts,amsmath,amssymb,epsfig,color,float,graphicx,verbatim}
\usepackage{algorithm,algorithmic}

\usepackage{hyperref,natbib}
\hypersetup{
	colorlinks   = true, 
	urlcolor     = blue, 
	linkcolor    = blue, 
	citecolor   = black 
}

\newtheorem{theorem}{Theorem}

\newtheorem{lemma}{Lemma}
\newtheorem{assumption}{Assumption}

\newcommand{\reals}{\mathbb{R}}
\newcommand{\E}{\mathbb{E}}

\newcommand{\sign}{\mathrm{sign}}

\newcommand{\bx}{\mathbf{x}}
\newcommand{\bw}{\mathbf{w}}

\newcommand{\bu}{\mathbf{u}}
\newcommand{\bv}{\mathbf{v}}

\newcommand{\Ocal}{\mathcal{O}}

\newcommand{\Dcal}{\mathcal{D}}

\newcommand{\norm}[1]{\|#1\|}

\newcommand{\secref}[1]{Sec.~\ref{#1}}

\renewcommand{\eqref}[1]{Eq.~(\ref{#1})}
\newcommand{\lemref}[1]{Lemma~\ref{#1}}

\newcommand{\thmref}[1]{Thm.~\ref{#1}}

\title{Gradient Methods Never Overfit On Separable Data}
\author{Ohad Shamir\\Weizmann Institute of Science}
\date{}

\begin{document}

\maketitle

\begin{abstract}
	A line of recent works established that when training linear predictors over separable data, using gradient methods and exponentially-tailed losses, the predictors asymptotically converge in direction to the max-margin predictor. As a consequence, the predictors asymptotically do not overfit. However, this does not address the question of whether overfitting might occur non-asymptotically, after some bounded number of iterations. In this paper, we formally show that standard gradient methods (in particular, gradient flow, gradient descent and stochastic gradient descent) \emph{never} overfit on separable data: If we run these methods for $T$ iterations on a dataset of size $m$, both the empirical risk and the generalization error decrease at an essentially optimal rate of $\tilde{\Ocal}(1/\gamma^2 T)$ up till $T\approx m$, at which point the generalization error remains fixed at an essentially optimal level of $\tilde{\Ocal}(1/\gamma^2 m)$ regardless of how large $T$ is. Along the way, we present non-asymptotic bounds on the number of margin violations over the dataset, and prove their tightness. 
\end{abstract}

\section{Introduction}

Motivated by empirical observations in the context of neural networks, there is considerable interest nowadays in studying the \emph{implicit bias} of learning algorithms. This refers to the fact that even without any explicit regularization or other techniques to avoid overfitting, the dynamics of the learning algorithm itself biases its output towards ``simple'' predictors that generalize well. 

In this paper, we consider the implicit bias in a well-known and simple setting, namely learning linear predictors ($\bx\mapsto \bx^\top \bw$) for binary classification with respect to linearly-separable data. In a recent line of works \citep{soudry2018implicit,ji2018risk,nacson2019convergence,ji2019refined,Dudik2020gradient}, it was shown that if we attempt to do this by minimizing the empirical risk (average loss) over a dataset, using gradient descent and any exponentially-tailed loss (such as the logistic loss), then the predictor asymptotically converges in direction to the max-margin predictor with respect to the Euclidean norm\footnote{Namely, $\arg\max_{\bw:\norm{\bw}_2=1}\min_i \bx_i^\top \bw$ for a given dataset $\bx_1,\bx_2,\ldots,\bx_m$.}. Since there are standard generalization bounds for predictors which achieve a large margin over the dataset, we get that asymptotically, gradient descent does not overfit, even if we just run it on the empirical risk function without any explicit regularization, and even if the number of iterations $T$ diverges to infinity. In follow-up works, similar results were also obtained for other gradient methods such as stochastic gradient descent and mirror descent \citep{nacson2019stochastic,gunasekar2018characterizing}, and for more complicated predictors such as linear networks, shallow ReLU networks, and linear convolutional networks \citep{ji2018gradient,ji2019polylogarithmic,gunasekar2018implicit}.

However, in practice the number of iterations $T$ is some bounded finite number. Thus, the asymptotic results above leave open the possibility that for a wide range of values of $T$, gradient methods do not achieve a good margin, and possibly overfit. Admittedly, many of these papers do provide finite-time guarantees for linear predictors, which all tend to have the following form: After $T$ iterations, the output of gradient descent $\bar{\bw}_T$ (normalized to have unit norm) satisfies
\begin{equation}\label{eq:jibound}
\norm{\bar{\bw}_T-\bw^*}~\leq~ \Ocal\left(\frac{1}{\log(T)}\right),
\end{equation}
where $\bw^*$ is the max-margin predictor, and the $\Ocal(\cdot)$ notation hides dependencies in the dataset size and the margin attained by $\bw^*$. However, such bounds do not satisfactorily address the problem above, since they decay extremely slowly with $T$. For example, suppose we ask how many iterations $T$ are needed, till we get a predictor which achieves some positive margin on all the data points, assuming there exists a unit-norm predictor $\bw^*$ achieving a margin of $\gamma$ (namely $\min_i \bx_i^\top \bw^*\geq \gamma$). If all we know is the bound in \eqref{eq:jibound}, we must require $\norm{\bar{\bw}_T-\bw^*}\leq \gamma$, which holds if $T>\exp(\Omega(1/\gamma)$ (and in fact, the actual required bound is much larger due to the hidden dependencies in the $\Ocal(\cdot)$ notation). For realistically small values of $\gamma$, this bound on $T$ is unacceptably large. Could it be that gradient methods do not overfit only after so many iterations? We note that in \citet{ji2019refined}, it is shown that \eqref{eq:jibound} is essentially tight, but this does not preclude the possibility that $\bar{\bw}_T$ does not overfit even before getting very close to $\bw^*$. 

In this paper, we show that this is indeed the case, and in fact, for \emph{any} number of iterations $T$, gradient methods do not overfit and essentially behave in the best manner we can hope for. Specifically, if the underlying data distribution is separable with margin $\gamma$, and we attempt to minimize the average of an exponential or logistic loss over a training set of size $m$, using standard gradient methods (gradient flow, gradient descent, or stochastic gradient descent), then the generalization error of the resulting predictor (with respect to the $0-1$ loss) is at most $\tilde{\Ocal}(1/\gamma^2T + 1/\gamma^2 m)$, up to constants and logarithmic factors. For $T\leq m$, this bound is $\tilde{\Ocal}(1/\gamma^2 T)$, which is essentially the same as the optimal upper bound on the empirical risk of the algorithm's output after $T$ iterations. In other words, both the generalization error and the empirical risk provably go down at the same (essentially optimal) rate. Once $T\geq m$, the empirical risk may further decrease to $0$, but the generalization error remains at $\tilde{\Ocal}(1/\gamma^2 m)$, which is well-known to be essentially optimal for any learning algorithm in this setting.

To prove these results, we also establish more refined, nonasymptotic bounds on the margins attained on the dataset, which are also applicable to other losses. In general, these bounds imply that for any  $\alpha\in [0,1]$, after $T$ iterations, the resulting predictor achieves a margin of $\Omega((1-\alpha)\gamma)$ on all but $\tilde{\Ocal}\left(\frac{1}{(\gamma^2 T)^\alpha}\right)$ of the data points. These bounds guarantee that as $T$ increases, for larger and larger portions of the dataset, the predictors achieve a margin of $\Omega(\gamma)$, implying good generalization properties. Furthermore, we prove a lower bound showing that such guarantees are essentially optimal. Finally, although we focus mostly on the exponential and logistic loss, we also discuss the applicability of our results to polynomially-tailed losses in the appendix.

Before continuing, we emphasize that the techniques we use in our upper bounds are not fundamentally new, and similar ideas were employed in previous analyses on the convergence to the max-margin predictor, such as in \citet{ji2018risk,ji2019polylogarithmic} (in fact, some of our results build on these analyses). However, we apply these techniques to a conceptually different question, about the non-asymptotic ability of gradient methods to attain some significant margin. In addition, since we only care about convergence to some large-margin predictor (as opposed to the max-margin predictor), our analysis can  be shorter and simpler. 

Finally, we note that polynomial-time, non-asymptotic guarantees on the generalization error of unregularized gradient methods were also obtained in \citet{ji2019polylogarithmic} and version 2 of \citet{ji2018risk}. However, the former is for nonlinear predictors, and the latter is for one-pass stochastic gradient descent, which is different than the algorithms considered here and where necessarily $T=m$. Moreover, the bounds in both papers have a worse polynomial dependence on the margin $\gamma$, compared to our results. 

The paper is structured as follows. In the next section, we define some useful notation, and formally describe our setting. In \secref{sec:flow}, we provide our positive results about the margin behavior and the generalization error, focusing on gradient flow (for which our analysis is the simplest and completely self-contained). In \secref{sec:gdsgd}, we show how similar results can be obtained for gradient descent and stochastic gradient descent. In \secref{sec:tightness}, focusing for concreteness on gradient descent, we show that our positive result on the margin behavior is essentially tight. In Appendix \ref{app:poly}, we briefly discuss how our results can be applied to polynomially-tailed losses, and their implications. Some technical proofs are provided in Appendix \ref{app:proofs}.

\section{Preliminaries}\label{sec:preliminaries}

We generally let boldfaced letters denote vectors. Given a positive integer $n$, we let $[n]$ be a shorthand for $\{1,\ldots,n\}$. Given a nonzero vector $\bw$, we let 
\[
\bar{\bw}:=\bw/\norm{\bw}
\]
denote its normalization to unit norm. We use the standard $\Ocal(\cdot)$ and $\Omega(\cdot)$ notation to hide constants, and $\tilde{\Ocal}(\cdot)$, $\tilde{\Omega}(\cdot)$ to hide constants and factors polylogarithmic in the problem parameters. $\log(\cdot)$ refers to the natural logarithm, and $\norm{\cdot}$ refers to the Euclidean norm. 

We consider datasets defined by a set of vectors $\bx_1,\ldots,\bx_m\in \reals^d$, and algorithms which attempt to minimize the empirical risk function, namely 
\[
\hat{L}(\bw):=\frac{1}{m}\sum_{i=1}^{m}\ell(\bx_i^\top\bw)
\]
where $\ell:\reals\mapsto\reals$ is some loss function\footnote{In the context of binary classification, it is customary to consider labeled data points $(\bx_1,y_1),\ldots,(\bx_m,y_m)$ and losses of the form $\bw\mapsto \ell(y_i \bx_i^\top \bw)$. However, for our purposes we can fold the binary label $y_i$ inside $\bx_i$, and treat this as a single vector.}. We will utilize the following two assumptions about the dataset and the loss:
\begin{assumption}\label{assump:separable}
	$\max_i \norm{\bx_i}\leq 1$, and the dataset is separable with margin $\gamma\in (0,1]$: Namely, there exists a unit vector $\bw$ s.t. $\min_i \bx_i^\top \bw \geq \gamma$. 
\end{assumption}

\begin{assumption}\label{assump:loss}
	$\ell$ is convex, monotonically decreasing, and has an inverse function\footnote{Namely, for any $z\in (0,\infty)$, there is a unique $p=\ell^{-1}(z)$ such that $\ell(p)=z$.} $\ell^{-1}$ on the interval $(0,\ell(0)]$.	
\end{assumption}

We note that assumption \ref{assump:separable} is without much loss of generality (it simply sets the scaling of the problem). Assumption \ref{assump:loss} implies that $\hat{L}$ is convex, and is satisfied for most classification losses. When instantiating our general results, we will focus for concreteness on the logistic loss $\ell(z)=\log(1+\exp(-z))$ and the exponential loss $\ell(z)=\exp(-z)$. However, our results can be applied to other losses as well.

In our proofs, we will make use of the following well-known facts (see for example \citet{nesterov2018lectures}): For a convex function $f$ on $\reals^d$, it holds for any vectors $\bu,\bv$ that $f(\bv)-f(\bu)\leq \nabla f(\bv)(\bv-\bu)$. Also, if $f$ is a function with $\mu$-Lipschitz gradients, then for any $\bu,\bv$, $f(\bv)\leq f(\bu)+\nabla f(\bu)(\bv-\bu)+\frac{\mu}{2}\norm{\bv-\bu}^2$. 

\section{Gradient Flow}\label{sec:flow}

In this section, we present positive results on the margin behavior and generalization error of gradient flow. Gradient flow is the standard continuous-time analogue of gradient descent. Although it cannot be implemented precisely in practice, it is a useful idealization of gradient descent, and the method in which it is easiest to present our analysis (the analysis for gradient descent in the next section is just a slight variation). Gradient flow produces a continuous trajectory of vectors $\bw(t)$, indexed by a time $t\geq 0$. It is defined by a starting point $\bw(0)$ (which will be the origin $\mathbf{0}$ in our case), and the differential equation 
\[
\frac{\partial}{\partial t}\bw(t)=-\nabla \hat{L}(\bw(t))~.
\]

We now present a general result about the margin behavior of gradient flow, which is applicable to general losses and any vector reached by gradient flow at any time point:
\begin{theorem}\label{thm:key}
	Under assumptions \ref{assump:separable} and \ref{assump:loss}, let $\bw\neq \mathbf{0}$ be some point reached by gradient flow, such that $\hat{L}(\bw)=\epsilon$ for some $\epsilon\in (0,\ell(0)]$. Then for any $p\in \left[\frac{\epsilon}{\ell(0)},1\right]$, for at least $(1-p)m$ of the indices $i\in [m]$, 
	\[
	\bx_i^\top \bar{\bw} ~>~ \frac{\gamma}{2}\cdot \frac{\ell^{-1}(\epsilon/p)}{\ell^{-1}(\epsilon)}~.
	\]
\end{theorem}
Intuitively, we expect the empirical risk $\hat{L}(\bw(t))$ to decay with $t$ (as we instantiate for specific losses later on). Computing the corresponding bounds on $\epsilon$ and plugging into the above, we can get guarantees on how many points in our dataset achieve a certain margin. Note that since $\ell^{-1}$ is monotonically decreasing, the margin lower bound in the theorem is always at most $\gamma/2$. Since under assumption \ref{assump:separable} the maximal margin is at least $\gamma$, this result cannot be used to recover the asymptotic convergence to the max-margin predictor, shown by previous results. However, as discussed in the introduction, this is not our focus here: We ask about the time to converge to some large-margin predictor which generalizes well, not necessarily the max-margin predictor. For that purpose, as we will see later, a margin lower bound of $\Omega(\gamma)$ is perfectly adequate.

To prove \thmref{thm:key}, we will need the following key lemma, which bounds the norm of the points along the trajectory of gradient flow, in terms of the value of $\hat{L}$. The proof is short and relies only on the convexity of $\hat{L}$:
\begin{lemma}\label{lem:key}
	Fix some $T\geq 0$, and let $\bw^*$ be any vector such that $\hat{L}(\bw^*)\leq \hat{L}(\bw(T))$. Then $\sup_{t\in [0,T]}\norm{\bw(t)}\leq 2\norm{\bw^*}$.
\end{lemma}
\begin{proof}
	By definition of gradient flow and the chain rule, we have $\frac{\partial}{\partial t}\hat{L}(\bw(t))=-\norm{\nabla \hat{L}(\bw(t))}^2\leq 0$, so $\hat{L}(\bw(t))$ is monotonically decreasing in $t$. As a result, for any $t$, $\hat{L}(\bw(t))\geq \hat{L}(\bw(T))\geq \hat{L}(\bw^*)$. By convexity of $\hat{L}$, is follows that
	\[
	\nabla \hat{L}(\bw(t))^\top(\bw(t)-\bw^*)~\geq~ \hat{L}(\bw(t))-\hat{L}(\bw^*)\geq 0~.
	\]
	Using this inequality, the definition of gradient flow and the chain rule, it follows that
	\[
	\frac{\partial}{\partial t}\norm{\bw(t)-\bw^*}^2~=~-2\nabla\hat{L}(\bw(t))^\top(\bw(t)-\bw^*)\leq 0~.
	\]
	Therefore, $\norm{\bw(t)-\bw^*}$ is monotonically decreasing in $t\in [0,T]$, so it is at most $\norm{\bw(0)-\bw^*}=\norm{\bw^*}$. Thus, by the triangle inequality, $\norm{\bw(t)}\leq \norm{\bw(t)-\bw^*}+\norm{\bw^*}\leq 2\norm{\bw^*}$. 
\end{proof}

\begin{proof}[Proof of \thmref{thm:key}]
Let $\bw_0,\norm{\bw_0}=1$ be a max-margin separator, so that $\min_i \bx_i^\top \bw_0\geq \gamma$. Define
$\bw^*:=\frac{\ell^{-1}(\epsilon)}{\gamma}\bw_0$, 
which has norm $\ell^{-1}(\epsilon)/\gamma$, and note that since $\ell^{-1}(\epsilon)$ is non-negative,
\[
\max_i \ell(\bx_i^\top \bw^*)~=~\ell\left(\min_i \frac{\ell^{-1}(\epsilon)}{\gamma}\bx_i^\top \bw_0\right)~\leq~ \ell(\ell^{-1}(\epsilon)\cdot 1)~=~ \epsilon~.
\]
This implies $\hat{L}(\bw^*)=\frac{1}{m}\sum_{i=1}^{m}\ell(\bx_i^\top\bw^*)\leq \epsilon$. Combined with \lemref{lem:key}, we get that
	\begin{equation}\label{eq:normbound}
	\norm{\bw}~\leq~2\norm{\bw^*}~=~\frac{2\ell^{-1}(\epsilon)}{\gamma}~.
	\end{equation}
Since $\hat{L}(\bw)=\epsilon$, we get that $\E_i[\ell(\bx_i^\top\bw)]\leq \epsilon$, where $i$ is uniformly distributed on $[m]$. By Markov's inequality and \eqref{eq:normbound}, it follows that
\[
p\geq \Pr_{i}\left(\ell(\bx_i^\top\bw) \geq \frac{\epsilon}{p}\right)
= \Pr_{i}\left(\bx_i^\top\bw \leq \ell^{-1}(\epsilon/p)\right)
= \Pr_{i}\left(\bx_i^\top\bar{\bw} \leq \frac{\ell^{-1}(\epsilon/p)}{\norm{\bw}}\right)
\geq
\Pr_{i}\left(\bx_i^\top\bar{\bw} \leq \frac{\gamma\ell^{-1}(\epsilon/p)}{2 \ell^{-1}(\epsilon)}\right)~.
\]
\end{proof}

For concreteness, let us now apply \thmref{thm:key} to the case of the exponential loss, $\ell(z)=\exp(-z)$. In order to get an interesting guarantee, we will utilize the following simple non-asymptotic guarantee on the decay of $\hat{L}(\bw(T))$ as a function of $T$:

\begin{lemma}\label{lem:Ldecayexp}
	Under assumption \ref{assump:separable}, if $\ell$ is the exponential loss, then $\hat{L}(\bw(T))\leq \frac{1}{\gamma^2 T}$ for any $T> 0$.
\end{lemma}
\begin{proof}
	Let $\bw_0$ be a max-margin unit vector, so that $\min_i \bx_i^\top \bw_0\geq \gamma$. By definition of gradient flow, the chain rule and Cauchy-Schwartz,
	\begin{align*}
	\frac{\partial}{\partial t} \hat{L}(\bw(t))~&=~ -\norm{\nabla \hat{L}(\bw(t))}^2
	~\leq~ -\left(\nabla \hat{L}(\bw(t))^\top\bw_0\right)^2
	~=~
	-\left(\frac{1}{m}\sum_{i=1}^{m}\ell'(\bx_i^\top \bw)\bx_i^\top \bw_0\right)^2\\
	&\leq~ -\left(\frac{1}{m}\sum_{i=1}^{m}\ell'(\bx_i^\top \bw)\gamma\right)^2
	~\stackrel{(*)}{=}~ -\left(\frac{1}{m}\sum_{i=1}^{m}\ell(\bx_i^\top \bw)\gamma\right)^2~=~ -\gamma^2 \hat{L}(\bw(t))^2~,
	\end{align*}
	where in $(*)$ we used the fact that when $\ell(z)=\ell'(z)=\exp(z)$ for all $z$. Consider now the function $f(t):=\gamma^2 t\cdot \hat{L}(\bw(t))$. We clearly have $f(0)=0$, and by differentiation and the inequality above, it is easily verified that $f'(t)\leq 0$ whenever $f(t)\geq 1$. Since $f(t)$ is continuous, we must have $f(t)\leq 1$ for all $t$, hence $\hat{L}(\bw(t))\leq \frac{1}{\gamma^2 t}$.
\end{proof}

Using this lemma, we get the following corollary of \thmref{thm:key} for the exponential loss:
\begin{theorem}\label{thm:marginquantiles}
Under assumption \ref{assump:separable}, if $\ell$ is the exponential loss, then for any $T> 1/\gamma^2$ and any $\alpha\in [0,1]$, it holds that $\bx_i^\top \bar{\bw}(T)> \frac{1-\alpha}{2}\cdot \gamma$ for at least $(1-(\gamma^2 T)^{-\alpha})m$ of the indices $i\in [m]$. 
\end{theorem}
\begin{proof}
	By \lemref{lem:Ldecayexp}, we know that at time $T$, $\hat{L}(\bw(T))=\epsilon$ for some $\epsilon \leq 1/\gamma^2 T<1=\ell(0)$ (which implies $\bw(T)\neq \mathbf{0}$). Applying \thmref{thm:key}, and noting that $\ell^{-1}(z)=\log(1/z)$ and $\ell(0)=1$, we get that for any $p\in [1/\gamma^2 T,1]$, for at least $(1-p)m$ of the indices $i$,
	\[
	\bx_i^\top \bar{\bw}_T~>~ \frac{\gamma}{2}\cdot \frac{\log(p/\epsilon)}{\log(1/\epsilon)}~=~
	\frac{\gamma}{2}\cdot \left(1-\frac{\log(1/p)}{\log(1/\epsilon)}\right)
	~\geq~	\frac{\gamma}{2}\cdot \left(1-\frac{\log(1/p)}{\log(\gamma^2 T)}\right)~.
	\]
	In particular, picking $p=(\gamma^2 T)^{-\alpha}$ for some $\alpha \in [0,1]$, the result follows.
\end{proof}

Note that if $T> m^{1/\alpha}/\gamma^2$, then $(\gamma^2 T)^{-\alpha} < \frac{1}{m}$, which implies that for \emph{all} $i\in [m]$, $\bx_i^\top \bar{\bw}(T)$ is at least $\frac{1-\alpha}{2}\cdot \gamma$. However, even for smaller values of $T$, the theorem provides guarantees on the margin attained on most points in the dataset. Moreover, using standard margin-based generalization bounds for binary classification, this theorem implies that the predictors returned by gradient flow achieve low generalization error, uniformly over all sufficiently large $T$:

\begin{theorem}\label{thm:exp}
Let $\Dcal$ be some distribution over $\{\bx:\norm{\bx}\leq 1\}\times \{-1,+1\}$, such that there exists a unit vector $\bw_0$ and $\gamma>0$ satisfying $\Pr_{(\bx,y)\sim \Dcal}(y\bx^\top \bw_0\geq \gamma)=1$. If we sample $m$ points $(\bx_1,y_1),\ldots,(\bx_m,y_m)$ i.i.d. from $\Dcal$, and run gradient flow on $\bw\mapsto \frac{1}{m}\sum_{i=1}^{m}\ell(y_i \bx_i^\top \bw)$, where $\ell$ is the exponential loss, then with probability at least $1-\delta$ over the sample, it holds for any time $T$ that
\[
\Pr_{(\bx,y)\sim D}(\sign(\bx^\top \bw(T))\neq y)~\leq~
\tilde{\Ocal}\left(\frac{1}{\gamma^2 T}+\frac{1}{\gamma^2 m}\right)~,
\]
where the $\tilde{\Ocal}$ notation hides universal constants and factors polylogarithmic in $\gamma^2 m$ and $1/\delta$. 
\end{theorem}

As discussed in the introduction, this is essentially the best behavior we can hope for (up to log factors): For $T\leq m$, the generalization error decays as $\tilde{\Ocal}(1/\gamma^2 T)$, which is also the bound on the empirical risk (see \lemref{lem:Ldecayexp}). Once $T\geq m$, the generalization error becomes $\tilde{\Ocal}(1/\gamma^2 m)$, and stays there regardless of how large $T$ is. 

\begin{proof}[Proof of \thmref{thm:exp}]
	The bound in the theorem is vacuous when $\gamma^2 m\leq 1$ or $\gamma^2 T\leq 1$, so we will assume without loss of generality that both quantities are larger than $1$. 
	
	Standard margin-based generalization bounds (e.g. \citet{mcallester2003simplified}) imply that in our setting, if we pick $m$ points i.i.d., then with probability at least $1-\delta$, any vector $\bw$ for which $\max_i y_i\bx_i^\top \bar{\bw}\geq \hat{\gamma}$ for all but $pm$ of the points satisfies
	\[
	\Pr_{(\bx,y)\sim D}(\sign(\bw^\top \bx)\neq y)~\leq~ p+\tilde{\Ocal}\left(\sqrt{p\cdot \frac{1}{\hat{\gamma}^2 m}}+\frac{1}{\hat{\gamma}^2 m}\right)
	~\leq~ \tilde{\Ocal}\left(p+\frac{1}{\hat{\gamma}^2 m}\right)~,
	\]
	where the $\tilde{\Ocal}$ hides universal constants and factors polylogarithmic in $1/\delta$ and $\gamma^2 m$. In particular, this can be applied uniformly for $\bw(T)$ for any $T$. 
	By \thmref{thm:marginquantiles}, we can substitute $p=(\gamma^2 T)^{-\alpha}$ and $\hat{\gamma}=(1-\alpha)\gamma/2$, to get that with probability at least $1-\delta$, 
	\begin{equation}\label{eq:gen1}
	\Pr_{(\bx,y)\sim D}(\sign(\bx^\top \bw(T))\neq y) ~\leq~ \tilde{\Ocal}\left(\frac{1}{(\gamma^2 T)^\alpha}+ \frac{1}{(1-\alpha)^2\gamma^2 m}\right)~.
	\end{equation}
	This holds for any $\alpha\in [0,1]$. In particular, if $\gamma^2 T\geq (\gamma^2 m)^2$, pick $\alpha=1/2$, in which case \eqref{eq:gen1} is at most
		\[
		\tilde{\Ocal}\left(\frac{1}{(\gamma^2 T)^{1/2}}+ \frac{1}{\gamma^2 m}\right)~=~
		\tilde{\Ocal}\left(\frac{1}{\gamma^2 m}\right)~,
		\]
	and if $\gamma^2 T < (\gamma^2 m)^2$, pick $\alpha=1-\frac{1}{2\log(\gamma^2 m)}$, in which case \eqref{eq:gen1} is at most
		\[
		\tilde{\Ocal}\left(\frac{(\gamma^2 T)^{1/2\log(\gamma^2 m)}}{\gamma^2 T}+ \frac{1}{\gamma^2 m}\right)
		~\leq~
		\tilde{\Ocal}\left(\frac{(\gamma^2 m)^{1/\log(\gamma^2 m)}}{\gamma^2 T}+ \frac{1}{\gamma^2 m}\right)~=~ \tilde{\Ocal}\left(\frac{1}{\gamma^2 T}+\frac{1}{\gamma^2 m}\right)~,
		\]
	where we used the fact that $z^{1/\log(z)}=\exp(\log(z)/\log(z))=\exp(1)$. Combining the two cases, the result follows.
\end{proof}

\section{Gradient Descent and Stochastic Gradient Descent}\label{sec:gdsgd}

Having discussed gradient flow, we show in this section how essentially identical results can be obtained for gradient descent and stochastic gradient descent. 

\subsection{Gradient Descent}

Gradient descent, which is probably the simplest and most well-known gradient method, optimizes $\hat{L}$ by initializing at some point $\bw_0$, and performing iterations of the form $\bw_{t+1}=\bw_t-\eta_t \nabla \hat{L}(\bw_t)$, where $\eta_1,\eta_2,\ldots$ are step size parameters. We will utilize the following standard assumption:
\begin{assumption}\label{assump:smooth}
	The derivative of $\ell$ is $\mu$-Lipschitz, and $0<\eta_t\leq 1/\mu$ for all $t$. 
\end{assumption}

Inspecting the analysis for gradient flow from the previous section, we note that we relied on the algorithm's structure only at two points: In \lemref{lem:key}, to bound the norm of the points along the trajectory, and in \lemref{lem:Ldecayexp}, to upper bound the values of $\hat{L}$. Fortunately, we can provide analogues of these two lemmas for gradient descent:

\begin{lemma}\label{lem:keygd}
	Under assumptions \ref{assump:separable},\ref{assump:loss} and \ref{assump:smooth}, fix some index $T\geq 1$, and let $\bw^*$ be any vector such that $\hat{L}(\bw^*)\leq \hat{L}(\bw_T)$. Then $\max_{t\in [T]}\norm{\bw_t}\leq 2\norm{\bw^*}$.
\end{lemma}
\begin{lemma}{(\citet[Theorem 3.1]{ji2018risk})}\label{lem:LdecaylogisticGD}
	If $\ell$ is the logistic loss, then under assumption \ref{assump:smooth}, gradient descent with step size $\eta=1$ satisfies $\hat{L}(\bw_T)~\leq~\frac{1}{T}+\frac{\log^2(T)}{2\gamma^2 T}$ for any $T$. 
\end{lemma}
The proof of \lemref{lem:keygd} (which is a slight variation on the proof of \lemref{lem:key}, appears in Appendix \ref{app:proofs}. 

With these lemmas, we can prove analogues of the theorems from the previous section, this time for gradient descent and for the logistic loss. The resulting bounds are identical up to constants and logarithmic factors:

\begin{theorem}\label{thm:keygd}
	Under assumptions \ref{assump:separable}, \ref{assump:loss} and \ref{assump:smooth}, let $\bw\neq \mathbf{0}$ be some point reached by gradient descent, such that $\hat{L}(\bw)=\epsilon$ for some $\epsilon\in (0,\ell(0)]$. Then for any $p\in \left[\frac{\epsilon}{\ell(0)},1\right]$, for at least $(1-p)m$ of the indices $i\in [m]$,
	\[
	\bx_i^\top \bar{\bw} ~>~ \frac{\gamma}{2}\cdot \frac{\ell^{-1}(\epsilon/p)}{\ell^{-1}(\epsilon)}~.
	\]
\end{theorem}

\begin{theorem}\label{thm:marginquantilesgd}
	Under assumption \ref{assump:separable}, if $\ell$ is the logistic loss, and we use a fixed step size of $\eta_t=1$ for all $t$, then for any $T>4$ such that $\frac{\log^2(T)}{\gamma^2 T}<\ell(0)$, and any $\alpha\in [0,1]$, the gradient descent iterates satisfy $\bx_i^\top \bar{\bw}_T> \frac{1-\alpha}{2}\cdot \gamma$ for at least $\left(1-2\left(\frac{\log^2(T)}{\gamma^2 T}\right)^{\alpha}\right)m$ of the indices $i\in [m]$. 
\end{theorem}

\begin{theorem}\label{thm:logisticgd}
	Let $\Dcal$ be some distribution over $\{\bx:\norm{\bx}\leq 1\}\times \{-1,+1\}$, such that there exists a unit vector $\bw_0$ and $\gamma>0$ satisfying $\Pr_{(\bx,y)\sim \Dcal}(y\bx^\top \bw_0\geq \gamma)=1$. If we sample $m$ points $(\bx_1,y_1),\ldots,(\bx_m,y_m)$ i.i.d. from $\Dcal$, and run gradient descent with fixed step sizes $\eta_t=1$ on $\bw\mapsto \frac{1}{m}\sum_{i=1}^{m}\ell(y_i \bx_i^\top \bw)$, where $\ell$ is the logistic loss, then with probability at least $1-\delta$ over the sample, it holds for any iteration $T$ that
	\[
	\Pr_{(\bx,y)\sim D}(\sign(\bx^\top \bw_T)\neq y)~\leq~
	\tilde{\Ocal}\left(\frac{1}{\gamma^2 T}+\frac{1}{\gamma^2 m}\right)~,
	\]
	where the $\tilde{\Ocal}$ notation hides universal constants and factors polylogarithmic in $\gamma^2 m$, $T$ and $1/\delta$. 
\end{theorem}

The results can be easily generalized to other step size strategies. The proofs are essentially identical to the proofs from the previous section, except that we use Lemma \ref{lem:keygd} and \ref{lem:LdecaylogisticGD} instead of Lemmas \ref{lem:key} and \ref{lem:Ldecayexp}. In particular, the proof of \thmref{thm:keygd} is identical to the proof of \thmref{thm:key}; the proof of \thmref{thm:marginquantilesgd} is nearly identical to the proof of \thmref{thm:marginquantiles} (and is provided in Appendix \ref{app:proofs} for completeness); and the proof of \thmref{thm:logisticgd} is identical to the proof of \thmref{thm:exp}, except that we use \thmref{thm:marginquantilesgd} and have some additional logarithmic factors which gets absorbed into the $\tilde{\Ocal}()$ notation.

\subsection{Stochastic Gradient Descent}

We now turn to discuss the stochastic gradient descent (SGD) algorithm, perhaps the main workhorse of modern machine learning methods. We consider the simplest version of SGD for minimizing $\hat{L}(\bw)=\frac{1}{m}\sum_{i=1}^{m}\ell(\bx_i^\top \bw)$: We initialize $\bw_1$ at the origin $\mathbf{0}$, and for any $t\geq 1$, define $\bw_{t+1}:=\bw_t-\eta\ell'(\bx_{i_t}^\top \bw_t)\bx_{i_t}$, where $i_t\in [m]$ is chosen independently and uniformly at random (so that in expectation, $\E_{i_t}[\bw_{t+1}]=\bw_t-\eta\nabla \hat{L}(\bw(t))$, similar to the gradient descent update). We assume that at the end of $T$ iterations, the algorithm returns the average of the iterates obtained so far, $\bv_T:=\frac{1}{T}\sum_{t=1}^{T}\bw_t$.

To avoid yet another (and more complicated) repetition of the analysis from the previous section, we will take a somewhat different route, focusing on the logistic loss and fixed step sizes $\eta_t=1$, which allows us to directly utilize some existing results in the literature to get a bound on the margin behavior of SGD (analogous to \thmref{thm:marginquantiles} for gradient flow and \thmref{thm:marginquantilesgd} for gradient descent). We note that the analysis can be easily generalized to other constant step sizes.

\begin{theorem}\label{thm:marginquantilesSGD}
Under assumption \ref{assump:separable}, if $\ell$ is the logistic loss, and we use step sizes $\eta_t=1$, then for any $T> 1/\gamma^2$ and any $\alpha\in [0,1]$, the SGD iterates satisfy $\bx_i^\top \bar{\bv}_T> \frac{1-\alpha}{5}\cdot \gamma$ for at least $\left(1-\delta\right)m$ of the indices $i\in [m]$, where $\delta$ is a nonnegative random variable (dependent on the randomness of SGD), whose expectation is at most $\frac{8+4\log^2(\gamma^2 T)}{3\gamma^2 T^{\alpha}}$. 
\end{theorem}

\begin{proof}
	We will utilize the following easily-verified facts about the logistic loss: It is non-negative, its gradient is $\frac{1}{4}$-Lipschitz, and its inverse is $\ell^{-1}(z):=\log(1/(\exp(z)-1))$, which is between $\log(1/z)$ and $\log(1/2z)$ for all $z\in [0,1]$~.
	
	Let $\bw_0,\norm{\bw_0}=1$ be a max-margin separator, so that $\min_i \bx_i^\top \bw_0\geq \gamma$, and define (similarly to the proof of \thmref{thm:exp}) $\bw^*:=\frac{\ell^{-1}(\epsilon)}{\gamma}\bw_0$, where $\epsilon$ will be chosen later. It is easily verified that $\max_i \ell(\bx_i^\top \bw^*)\leq \epsilon$, and therefore $\hat{L}(\bw^*)\leq \epsilon$. 
	Since $\ell$ is non-negative and with a $\frac{1}{4}$-Lipschitz derivative, we can use Theorem 14.13 from \citet{shalev2014understanding} on the convergence of SGD for such losses to get
	\[
	\E[\hat{L}(\bv_T)]~\leq~\frac{1}{1- \frac{1}{4}}\left(\hat{L}(\bw^*)+\frac{\norm{\bw^*}^2}{2T}\right)~\leq~ \frac{4}{3}\left(\epsilon+\frac{(\ell^{-1}(\epsilon))^2}{2\gamma^2 T}\right)~.
	\]
	In particular, picking $\epsilon = \frac{1}{\gamma^2 T}$, we get 
	\[
	\E[\hat{L}(\bv_T)]~\leq~ \frac{4}{3\gamma^2 T}\left(1+\frac{(\ell^{-1}(1/\gamma^2 T))^2}{2}\right)
	~\leq~ \frac{4}{3\gamma^2 T}\left(1+\frac{1}{2}\log^2(\gamma^2 T)\right)~.
	\]
	To simplify notation, define $\Delta(T,\gamma)$ to be the expression in the right-hand side above. We also note that the left-hand side equals $\E_{\bv_T,i}[\ell(\bx_i^\top \bv_T)]$, where $i$ is uniformly distributed in $[m]$. Thus, by Markov's inequality, for any $p>0$, we have
	\[
	p~\geq~\Pr_{i,\bv_T}\left(\ell(\bx_i^\top\bv_T) \geq \frac{\Delta(T,\gamma)}{p}\right)~=~
	\Pr_{i,\bv_T}\left(\bx_i^\top\bv_T \leq \ell^{-1}\left( \frac{\Delta(T,\gamma)}{p}\right)\right)~.
	\]
	According to Theorem 2.1 in version 2 of \citet{ji2018risk}, the iterates of SGD on the logistic loss satisfy deterministically $\max_{t\in [T]}\norm{\bw_t}\leq \frac{2\log(T)}{\gamma}+2$, so by Jensen's inequality, $\norm{\bv_T}\leq\frac{1}{T}\sum_{t=1}^{T}\norm{\bw_t}\leq \frac{2\log(T)}{\gamma}+2$. Combining this with the displayed equation above,
	we get that
	\[
	p~\geq~ \Pr_{i,\bv_T}\left(\bx_i^\top\bar{\bv}_T \leq \frac{1}{2+2\log(T)/\gamma}\cdot\ell^{-1}\left( \frac{\Delta(T,\gamma)}{p}\right)\right)~.
	\]
	Choosing $p=2T^{1-\alpha}\Delta(T,\gamma)=\frac{8+4\log^2(\gamma^2 T)}{3\gamma^2 T^{\alpha}}$ for some $\alpha\in [0,1]$, and substituting into the above, we get that
	\begin{align*}
	\frac{8+4\log^2(\gamma^2 T)}{3\gamma^2 T^{\alpha}}~&\geq~\Pr_{i,\bv_T}\left(\bx_i^\top\bar{\bv}_T ~\leq~ \frac{1}{2+2\log(T)/\gamma}\cdot\ell^{-1}\left(\frac{T^{\alpha-1}}{2}\right)\right)\\
	&\geq~ \Pr_{i,\bv_T}\left(\bx_i^\top\bar{\bv}_T ~\leq~ \frac{1}{2+2\log(T)/\gamma}\cdot\log\left(T^{1-\alpha}\right)\right)\\
	&=~ \Pr_{i,\bv_T}\left(\bx_i^\top\bar{\bv}_T ~\leq~ \frac{(1-\alpha)\gamma}{2}\cdot\frac{\log(T)}{\gamma+\log(T)}\right)~.
	\end{align*}
	Since $T>\frac{1}{\gamma^2}\geq 1$, we have $\frac{\log(T)}{\gamma+\log(T)}\geq \frac{\log(T)}{1+\log(T)}\geq \frac{\log(2)}{1+\log(2)}> \frac{2}{5}$. Plugging into the above, we get
	\[
	\frac{8+4\log^2(\gamma^2 T)}{3\gamma^2 T^{\alpha}}~\geq~\Pr_{i,\bv_T}\left(\bx_i^\top\bar{\bv}_T ~\leq~ \frac{(1-\alpha)\gamma}{5}\right)~.
	\]
	To complete the proof, we note that if we let $A_i$ denote the event that $\bx_i^\top\bar{\bv}_T\leq \frac{(1-\alpha)\gamma}{5}$, and $\mathbf{1}_{A_i}$ the indicator function of the event $A_i$, then the above implies
	\[
	\frac{8+4\log^2(\gamma^2 T)}{3\gamma^2 T^{\alpha}}~\geq~\E_{i,\bv_T}[\mathbf{1}_{A_i}]~=~\E_{\bv_T}\E_i[\mathbf{1}_{A_i}]~=~\E_{\bv_T}\left[\frac{1}{m}\sum_{i=1}^{m}\mathbf{1}_{A_i}\right].
	\]
	Thus, letting $\delta=\frac{1}{m}\sum_{i=1}^{m}\mathbf{1}_{A_i}$, the theorem follows.
\end{proof}

Using this theorem, we get the following generalization error bound for SGD, which is a direct analogue of the error bounds we obtained for gradient flow and gradient descent:

\begin{theorem}\label{thm:sgdbound}
	Let $\Dcal$ be some distribution over $\{\bx:\norm{\bx}\leq 1\}\times \{-1,+1\}$, such that there exists a unit vector $\bw_0$ and $\gamma>0$ satisfying $\Pr_{(\bx,y)\sim \Dcal}(y\bx^\top \bw_0\geq \gamma)=1$. Suppose we sample $m$ points $(\bx_1,y_1),\ldots,(\bx_m,y_m)$ i.i.d. from $\Dcal$, and run SGD (with step sizes $\eta_t=1$) on $\bw\mapsto \frac{1}{m}\sum_{i=1}^{m}\ell(y_i \bx_i^\top \bw)$, where $\ell$ is the logistic loss. Then for any $T$, 
	\[
	\E\left[\Pr_{(\bx,y)\sim D}(\sign(\bx^\top \bv_T)\neq y)\right]~\leq~
	\tilde{\Ocal}\left(\frac{1}{\gamma^2 T}+\frac{1}{\gamma^2 m}\right)~,
	\]
	where the $\tilde{\Ocal}$ notation hides universal constants and factors polylogarithmic in $T,m,1/\gamma$, and the expectation is over the randomness of the SGD algorithm.
\end{theorem}
\begin{proof}
	Using identical arguments as in the proof of \thmref{thm:exp} (except using \thmref{thm:marginquantilesSGD} instead of \thmref{thm:marginquantiles}, and the fact that high-probability bounds imply a bound on the expectation), we get that
	\begin{equation}\label{eq:gensgd}
	\E\left[\Pr_{(\bx,y)\sim D}(\sign(\bx^\top \bw(T))\neq y)\right] ~\leq~ \tilde{\Ocal}\left(\frac{1}{\gamma^2 T^\alpha}+ \frac{1}{(1-\alpha)^2\gamma^2 m}\right)~.
	\end{equation}
	This holds for any $\alpha\in [0,1]$. In particular, if $T\geq m^2$, pick $\alpha=1/2$, in which case \eqref{eq:gensgd} is at most
		\[
		\tilde{\Ocal}\left(\frac{1}{\gamma^2 T^{1/2}}+ \frac{1}{\gamma^2 m}\right)~=~
		\tilde{\Ocal}\left(\frac{1}{\gamma^2 m}\right)~,
		\]
	and if $T < m^2$, pick $\alpha=1-\frac{1}{\log( m^2)}$, in which case \eqref{eq:gensgd} is at most
		\[
		\tilde{\Ocal}\left(\frac{T^{1/\log(m^2)}}{\gamma^2 T}+ \frac{1}{\gamma^2 m}\right)
		~\leq~
		\tilde{\Ocal}\left(\frac{(m^2)^{1/\log(m^2)}}{\gamma^2 T}+ \frac{1}{\gamma^2 m}\right)~=~\tilde{\Ocal}\left(\frac{1}{\gamma^2 T}+\frac{1}{\gamma^2 m}\right)~.
		\]
	Combining the two cases, the result follows.
\end{proof}

Finally, we remark that \thmref{thm:sgdbound} only bounds the expectation of the error probability of $\bv_T$. It is likely that using more sophisticated concentration tools, one can obtain a high-probability bound. However, this would require a more involved analysis, and is left for future work.

\section{Tightness}\label{sec:tightness}

In the previous sections, we provided bounds on the margin behavior of gradient flow, gradient descent and SGD, which all have the following form (ignoring constants and log factors): After $T$ iterations, we get a predictor which achieves $\Omega(\gamma)$ margin on at least $(1-(\gamma^2 T)^{-\alpha})m$ of the $m$ data points, where $\alpha\in (0,1)$ is a constant arbitrarily close to one. It is natural to ask whether this result can be improved.

In this section, we show that this result is essentially tight: After $T$ iterations, it is impossible to guarantee any positive margin on more than $(1-p)m$ of the points, when $p$ is much less than $(\gamma^2 T)^{-1}$. For concreteness, we prove this for gradient descent and the logistic loss, although the analysis can be extended to other gradient methods and losses:

\begin{theorem}\label{thm:lowbound}
	For any positive integers $m,T$ and any $\gamma\in \Big(0,\frac{1}{8}\Big]$, there exists a dataset of $m$ points in $\reals^2$ satisfying assumption \ref{assump:separable}, such that gradient descent using the logistic loss and any step size $\eta\leq 1$ must satisfy $\bx_i^\top \bw_T\leq 0$ for at least $\lfloor\frac{m}{26\max\{1,\gamma^2 T\}}\rfloor$ of the data points.
\end{theorem}

Since $\bx_i^\top\bw_T\leq 0$ translates to $\ell(\bx_i^\top\bw_T)\geq \Omega(1)$, the theorem also implies that the $\tilde{\Ocal}(1/\gamma^2 T)$ bounds on the empirical risk shown earlier are tight up to logarithmic factors. It also implies that the percentage of misclassified points ($\bx_i^\top \bw_T\leq 0$) cannot decrease at a better rate. In addition, the lower bound implies that if we want to get any positive margin on all $m$ data points, the number of iterations $T$ must be at least $\Omega(m/\gamma^2)$ in the worst case.

A lower bound related to ours appears in \citet{ji2019refined}, where the authors show that if $\bw^*$ is the max-margin unit-norm predictor, then  $\norm{\bar{\bw}_T-\bw^*}\geq\Omega(\log(m)/\log(T))$. This translates to a $T\geq \Omega(m)$ requirement to get a non-trivial guarantee on the direction of $\bar{\bw}_T$. However, this is a somewhat different objective than ours, and moreover, their lower bound does not specify the dependence on the margin parameter $\gamma$.

\begin{proof}[Proof of \thmref{thm:lowbound}]
	
	We will assume without loss of generality that $T\geq \frac{1}{\gamma^2}$ (or equivalently, $\gamma^2 T\geq 1$), in which case the lower bound we need to prove is $\lfloor\frac{m}{26\gamma^2 T}\rfloor$. Otherwise, if $T<\frac{1}{\gamma^2}$, we can simply apply the construction below with the larger margin parameter $\hat{\gamma}:=\frac{1}{\sqrt{T}}$, which by definition satisfies $T\geq \frac{1}{\hat{\gamma}^2}$, and uses a dataset separable with margin $\hat{\gamma}\geq \gamma$, so assumption \ref{assump:separable} still holds with margin parameter $\gamma$. 
	
	We will also assume without loss of generality that $\lfloor\frac{m}{26\gamma^2 T}\rfloor>0$ (otherwise the theorem trivially holds), and fix $\epsilon:=\frac{1}{m}\left\lfloor\frac{m}{26\gamma^2 T}\right\rfloor$ (which is positive and less than $\frac{1}{26\gamma^2 T}\leq \frac{1}{26}$). 
	
	\begin{figure}[t]
	\centering
	\includegraphics[trim=4cm 9cm 4cm 9.1cm,clip=true,scale=0.8]{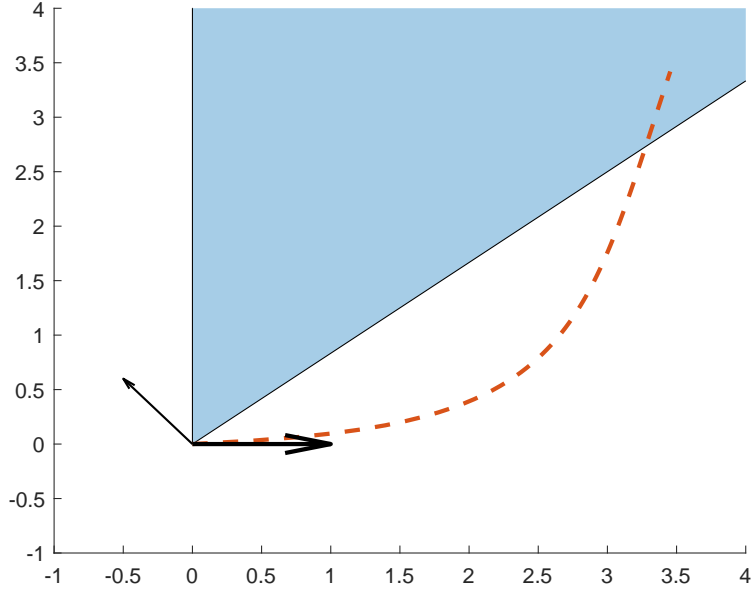}
	\caption{Illustration of the proof of \thmref{thm:lowbound}, for $\gamma=\frac{1}{5}$. The dataset consists of 90\% points at $(1,0)$ (thick black arrow) and 10\% points at $(-1/2,3\gamma)$ (thin black arrow). The gradient descent trajectory (with $\eta=1$, starting from the origin) is the dotted red line, and the shaded blue region are the vectors which achieves a positive margin on more than 90\% of the data points (or equivalently for this construction, get a positive margin on all data points). Initially, the gradient descent trajectory is mostly influenced by the points at $(1,0)$, and only once a sufficiently large margin is achieved on them, the influence of the few points at $(-1/2,3\gamma)$ begins to manifest, and the trajectory curves towards the blue region. Best viewed in color.}
	\label{fig:lowbound}
\end{figure}
	
	Consider a dataset consisting of a majority group of $(1-\epsilon)m$ points of the form $(1,0)$, and a minority group of $\epsilon m$ points of the form $\left(-\frac{1}{2},3\gamma\right)$ (see Figure \ref{fig:lowbound} for a sketch of the construction and proof idea). It is easily verified that all these points are contained in the unit ball, and that the vector $\bw=\left(\gamma,\frac{1}{2}\right)$ satisfies $\bx^\top \bar{\bw}\geq \bx^\top \bw\geq \gamma$ for any $\bx$ in the dataset. Thus, assumption \ref{assump:separable} holds. Moreover, the empirical risk function equals
	\[
	\hat{L}(\bw)~=~ \hat{L}(w(1),w(2))~=~\left(1-\epsilon\right)\cdot\ell(w(1))+\epsilon\cdot \ell\left(-\frac{1}{2}\cdot w(1)+3\gamma\cdot w(2)\right)~,
	\]
	where $\ell(z):=\log(1+\exp(-z))$ is the logistic loss. 
	
	Suppose by contradiction that $\bx^\top \bw_T\leq 0$ for less than $\epsilon m$ of the points $\bx$ in the dataset. Since there are $\epsilon m$ points of the form $\bx=\left(-\frac{1}{2},3\gamma\right)$, this implies that all these points are positively correlated with $\bw_T$, that is $\bx^\top \bw_T=-\frac{1}{2}\cdot w_T(1)+3\gamma\cdot w_T(2)> 0$. The lemma below implies that this cannot happen unless $T\geq \frac{19}{480\gamma^2 \epsilon}$. But since $\frac{19}{480\gamma^2 \epsilon}\geq \frac{19}{480\gamma^2 (1/26\gamma^2 T)}>T$, we get $T>T$, a contradiction. 
	
\begin{lemma}\label{lem:lowboundtech}
	Fix some $\epsilon$ and $\gamma$ in $\Big(0,\frac{1}{8}\Big]$. Consider gradient descent on the bivariate function
	\[
	\hat{L}(r,s):=~(1-\epsilon)\cdot\ell(r)+\epsilon\cdot\ell\left(-\frac{1}{2}r+3\gamma s\right)~,
	\]
	where $\ell$ is the logistic loss, using any step size $\eta\leq 1$, starting from $r_1=s_1=0$ and producing a sequence of iterates $(r_2,s_2),(r_3,s_3),\ldots$. Then if $-\frac{1}{2}r_t+3\gamma s_t>0$, we must have $t\geq \frac{19}{480\gamma^2 \epsilon}$. 
\end{lemma}
The proof of the lemma is based on a technical calculation, and appears in Appendix \ref{app:proofs}. 

\end{proof}

\subsection*{Acknowledgements}

This research is supported in part by European Research Council (ERC) grant 754705. We thank Matus Telgarsky and Ronen Eldan for very helpful discussions.

\bibliographystyle{plainnat}
\bibliography{bib}

\appendix

\section{Polynomially-Tailed Losses}\label{app:poly}

In our paper, we focused mostly on the exponential and logistic losses, which both have exponentially decaying tails. \citet{nacson2019convergence} showed that such a tail is important to get asymptotic convergence to a max-margin solution: If $\ell(z)$ decays polynomially with $z$, then the convergence may not be to a max-margin solution. 

Despite this, in the spirit of our previous results, one may still conjecture that we converge to some large-margin solution (though not a max-margin one). Recently,  \citet[Remark 8]{Dudik2020gradient} showed this holds, albeit in a weak sense: If the loss $\ell(z)$ decays as $z^{-b}$ for some $b>0$, then the asymptotic margin over the dataset is  $\Omega(m^{-1/(b+1)})$, and this is tight in general. 

Using our techniques, we can provide the following more refined margin bound:
\begin{theorem}
	Under assumptions \ref{assump:separable}, \ref{assump:loss}, \ref{assump:smooth}, let $\ell$ be a loss satisfying $\ell(z)=z^{-b}$ for all $z\geq 1$. Let $\bw$ be some point reached by gradient descent, such that $\hat{L}(\bw)=\epsilon$ for some $\epsilon \in (0,\ell(0))$. Then for any $p\in [\epsilon,1]$, it holds that for at least $(1-p)m$ of the indices $i\in [m]$ that
	\[
	\bx_i^\top \bar{\bw}~>~ \frac{\gamma}{2}\cdot p^{1/b}~.
	\] 
\end{theorem}
\begin{proof}
	By definition, $\ell^{-1}(z) = z^{-1/b}$ for any $z\in (0,1]$. Plugging this into \thmref{thm:keygd}, and noting that $\ell(0)\geq \ell(1)=1$, it follows that for any $p\in [\epsilon,1]$, for at least $(1-p)m$ of the indices $i\in [m]$,
	\[
	\bx_i^\top \bar{\bw}~>~ \frac{\gamma}{2}\cdot \frac{(\epsilon/p)^{-1/b}}{\epsilon^{-1/b}}~=~\frac{\gamma}{2}\cdot p^{1/b}~.
	\] 	
\end{proof}

Assuming gradient descent converges to a globally minimal value, we have $\epsilon\rightarrow 0$. As a result, after sufficiently many iterations, we can pick $p$ slightly below $1/m$ (say $1/2m$), and get a margin bound of the form $\Omega(\gamma\cdot m^{-1/k})$ on all of the training examples (since getting a margin violation on less than $1/m$ of $m$ examples is equivalent to no margin violations). This essentially recovers the result of \citet{Dudik2020gradient} (up to the small difference of having $b$ in the exponent instead of $b+1$). However, our theorem covers a wider regime, since we can pick other values of $p$. For example, by picking $p$ to be any constant, the theorem implies that we attain $\Omega(\gamma)$ margin on a constant portion of the training data (which can be arbitrarily close to $1$). 

Finally, we note that by upper bounding $\epsilon$ as a function of the number of iterations, the theorem above can be used to quantify how many iterations are required to achieve a certain margin on a certain percentage of the data points, as well as derive margin-based generalization bounds, in a manner completely analogous to the results in \secref{sec:gdsgd}. Since polynomially-tailed losses are not commonly used in practice, we do not further pursue this here.

\section{Additional Proofs}\label{app:proofs}

\subsection{Proof of \lemref{lem:keygd}}\label{app:lemkeygd}

	We first argue that $\hat{L}(\bw_t)$ is monotonically decreasing in $t$. To see this, note that by the assumption that the derivative of $\ell$ (and hence the gradient of $\hat{L}$) is $\mu$-Lipschitz, we have
	\begin{align*}
	\hat{L}(\bw_{t+1})~&=~ \hat{L}(\bw_t-\eta_t\nabla \hat{L}(\bw_t))~\leq~
	\hat{L}(\bw_t)-\eta_{t}\norm{\nabla \hat{L}(\bw_t)}^2+\frac{\mu}{2}\eta_{t}^2\norm{\nabla \hat{L}(\bw_t)}^2\\
	&=~\hat{L}(\bw_t)-\eta_{t}\norm{\nabla \hat{L}(\bw_t)}^2\left(1-\frac{\mu}{2}\eta_{t}\right)~,
	\end{align*}
	which implies that $\hat{L}(\bw_{t+1})\leq \hat{L}(\bw_t)$ whenever $\eta_{t}\leq 2/\mu$, which is indeed assumed. This also implies that $\hat{L}(\bw^*)\leq \hat{L}(\bw_T)\leq \hat{L}(\bw_t)$ for all $t$. 
	
	Next, we argue that $\norm{\bw_t-\bw^*}$ is monotonically decreasing in $t$. To see this, note that by smoothness of $\hat{L}$ and the monotonicity property above, for any $t<T$,
	\begin{align*}
	\hat{L}(\bw^*)-\hat{L}(\bw_t)~&\leq~ \hat{L}(\bw_{T})-\hat{L}(\bw_t) ~\leq~ \hat{L}(\bw_{t+1})-\hat{L}(\bw_t)
	~=~ \hat{L}(\bw_t-\eta_{t}\nabla \hat{L}(\bw_t))-\hat{L}(\bw_t)\\
	&\leq -\eta_{t}\norm{\nabla \hat{L}(\bw_t)}^2+\frac{\mu}{2}\eta_{t}^2\norm{\nabla \hat{L}(\bw_t)}
	~=~ -\eta_{t}\left(1-\frac{\eta_{t} \mu}{2}\right)\norm{\nabla \hat{L}(\bw_t)}^2\\
	&~\leq~
	-\frac{\eta_{t}}{2}\norm{\nabla \hat{L}(\bw_t)}^2~,
	\end{align*}
	where in the last step we used the fact that $\eta_{t} \leq 1/\mu$, hence $1-\frac{\eta_{t} \mu}{2}\leq \frac{1}{2}$. Overall, we get that 
	\[
	\norm{\nabla \hat{L}(\bw_t)}^2~\leq~ \frac{2}{\eta_{t}}(\hat{L}(\bw_t)-\hat{L}(\bw^*)~.
	\]
	Employing this inequality, together with $\hat{L}(\bw_t)-\hat{L}(\bw^*)\leq \nabla \hat{L}(\bw_t)^\top (\bw_t-\bw^*)$ (which follows from convexity of $\hat{L}$), we have the following:
	\begin{align*}
	\norm{\bw_{t+1}-\bw^*}^2-\norm{\bw_t-\bw^*}~&=~
	\norm{\bw_{t}-\eta_{t}\nabla \hat{L}(\bw_t)-\bw^*}^2-\norm{\bw_t-\bw^*}\\
	&=~ -2\eta_{t}\nabla \hat{L}(\bw_t)^\top (\bw_t-\bw^*)+\eta_{t}^2\norm{\nabla \hat{L}(\bw_t)}^2\\
	&\leq~ -2\eta_{t}(\hat{L}(\bw_t)-\hat{L}(\bw^*))+\eta_{t}^2\cdot \frac{2}{\eta_{t}}(\hat{L}(\bw_t)-\hat{L}(\bw^*)~=~0~.
	\end{align*}
	
	To complete the proof, note that since $\norm{\bw_t-\bw^*}$ is monotonically decreasing, then for any $t$, it is at most $\norm{\bw_1-\bw^*}=\norm{\bw^*}$. Therefore, $\norm{\bw_t}\leq \norm{\bw_t-\bw^*}+\norm{\bw^*}\leq 2\norm{\bw^*}$.

\subsection{Proof of \thmref{thm:marginquantilesgd}}

By \lemref{lem:LdecaylogisticGD}, we know that at iteration $T$, $\hat{L}(\bw(T))=\epsilon$ for some 
\[
\epsilon~\leq~ \frac{1}{T}+\frac{\log^2(T)}{2\gamma^2 T}
~\leq~ \frac{1}{\gamma^2 T}\left(1+\frac{1}{2}\log^2(T)\right)~\leq~ \frac{\log^2(T)}{\gamma^2 T}~,
\]
where we used the facts that $\gamma$ must be at most $1$ and $1\leq \frac{1}{2}\log^2(T)$ (since $T>4$). By the theorem assumptions, it follows that $\epsilon < \ell(0)$, so we must have $\bw(T)\neq \mathbf{0}$. Applying \thmref{thm:keygd} (noting that $\ell$ has $\frac{1}{4}$-Lipschitz gradients and that $\ell^{-1}(z):=\log(1/(\exp(z)-1))$, which is between $\log(1/z)$ and $\log(1/2z)$ for all $z\in [0,1]$), we get that for any $p\in [\frac{\epsilon}{\ell(0)},1]$, for at least $(1-p)m$ of the indices $i$,
\[
\bx_i^\top \bar{\bw}_T~>~ \frac{\gamma}{2}\cdot \frac{\log(p/2\epsilon)}{\log(1/\epsilon)}~=~
\frac{\gamma}{2}\cdot \left(1-\frac{\log(2/p)}{\log(1/\epsilon)}\right)
~.
\]
In particular, picking $p=2\epsilon^\alpha$ for some $\alpha \in [0,1]$ (which satisfies $p\geq \frac{\epsilon}{\ell(0)}=\frac{\epsilon}{\exp(2)}$), and noting that $p\leq 2\left(\frac{\log^2(T)}{\gamma^2 T}\right)^{\alpha}$, the result follows.

\subsection{Proof of \lemref{lem:lowboundtech}}

We will utilize the following easily-verified facts about the logistic loss $\ell$: $-1\leq \ell'(z)<0$ for all $z$, and $\ell'(z)<-\frac{1}{4}$ for all $z\leq 1$. Also, by definition of gradient descent, we have the following: $r_1=s_1=0$, and
\begin{align*}
r_{t+1}~&=~r_t-(1-\epsilon)\eta\cdot\ell'(r_t)+\frac{\epsilon\eta}{2}\cdot\ell'\left(-\frac{1}{2}r_t+3\gamma s_t\right)\\
s_{t+1}~&=~s_t-3\gamma \epsilon \eta\cdot\ell'\left(-\frac{1}{2}r_t+3\gamma s_t\right)~.
\end{align*}

The proof relies on the following key lemma:
\begin{lemma}\label{lem:inittime}
	There exists an iteration index $T\leq \frac{32}{5\eta}+1$ for which 
	\begin{enumerate}
		\item $r_t\geq\frac{15}{16}$ for all $t\geq T$
		\item $s_T\leq \frac{3}{10}$
		\item $-\frac{1}{2}r_t+3\gamma s_t\leq 0$ for all $t\leq T$.
	\end{enumerate}
\end{lemma}
\begin{proof}
	We start with the first item. If $r_t\leq 1$, then by the facts mentioned earlier and the assumption on $\epsilon$,
	\[
	r_{t+1}~\geq~ r_t+(1-\epsilon)\eta \cdot \frac{1}{4}-\frac{\epsilon \eta}{2}
	~=~ r_t+\frac{\eta}{4}\left(1-3\epsilon\right)~\geq~
	r_t+\frac{5\eta}{32}~.
	\]
	Since we start at $r_1=0$, we get that as long as $r_t\leq 1$, $r_t$ increases in increments of at least $5\eta/32$. Thus, there is some index $T\leq \frac{32}{5\eta}+1$ at which $r_T\geq 1$ for the first time, and $r_t$ is monotonically increasing for all $t\leq T$. Once $r_t\geq 1$, it cannot decrease to $\frac{15}{16}$ or below in the following iteration, since by the update equation,
	\[
	r_{t+1}~\geq~ r_t+0-\frac{\epsilon\eta}{2} ~\geq~ r_t-\frac{1}{16}~\geq~ \frac{15}{16}~,
	\]
	and if $r_{t+1}\leq 1$, it must monotonically increase again as shown earlier. This establishes the first item in the lemma.
	
	Turning to the second item, we note that $s_1= 0$ and for any $t$, $s_{t+1}\leq s_t+3\gamma \epsilon \eta \leq s_t+3\cdot \frac{1}{8}\cdot \frac{1}{8}\cdot \eta= s_t+\frac{3\eta}{64}$. Since $T\leq \frac{32}{5\eta}+1$, it follows that 
	\[
	s_T~\leq~ \frac{3\eta}{64}(T-1)~\leq~ \frac{3\eta}{64}\cdot\frac{32}{5\eta}~=~ \frac{3}{10}~.
	\]
	
	Turning to the third item in the lemma, define for simplicity $u_t:=-\frac{1}{2}r_t+3\gamma s_t$. Thus, we need to show $u_t\leq 0$ for all $t\leq T$. This trivially holds for $t=1$. For any $t<T$, by the update equations for $r_t,s_t$, 
	\begin{align*}
	u_{t+1}~&=~ u_t+\frac{(1-\epsilon)\eta}{2} \cdot \ell'(r_t)-\left(\frac{1}{4}+9\gamma^2\right)\epsilon\eta\cdot \ell'(u_t)\\
	&\stackrel{(*)}{\leq}~ u_t-\frac{(1-\epsilon)\eta}{8} +\left(\frac{1}{4}+\frac{9}{64}\right)\epsilon \eta~=~
	u_t-\frac{\eta}{8}\left(1-\frac{33}{8}\epsilon\right)~\stackrel{(**)}{<}~ u_t~.
	\end{align*}
	where in $(**)$ we used the assumption $\epsilon \leq \frac{1}{8}$, and in $(*)$ we used the facts that $\gamma\leq \frac{1}{8}$, that $-1\leq\ell'(z)\leq 0$ for any $z$, and that since $t<T$, $r_t< 1$ (see the definition of $T$ above), hence $\ell'(r_t)<-\frac{1}{4}$. Overall, we get that $u_t$ is monotonically decreasing for all $t\leq T$. But we have $u_1=-\frac{1}{2}r_1+3\gamma s_1=0$, hence $u_t$ must be non-positive for all $t\leq T$ as required.
\end{proof}

With this lemma at hand, we turn to prove \lemref{lem:lowboundtech}. According to the lemma, $-\frac{1}{2}r_t+3\gamma s_t> 0$ can only occur for $t>T$ (in which case, $r_t\geq \frac{15}{16}$). This requires that 
\[
s_t ~>~ \frac{1}{6\gamma}r_t ~\geq~ \frac{15}{96\gamma}~.
\]
However, by the lemma, we have $s_T\leq \frac{3}{10}$, and by the update rule for $s_t$, $s_{t+1} \leq s_t+3\gamma \epsilon \eta\leq s_t+3\gamma\epsilon$. Thus, the number of additional iterations (after iteration $T$) required to make $s_t > \frac{15}{96\gamma}$ is at least $\frac{15/96\gamma-3/10}{3\gamma \epsilon}~=~ \frac{5}{96\gamma^2\epsilon}\left(1-\frac{96}{50}\gamma\right)\geq \frac{5}{96\gamma^2 \epsilon}\left(1-\frac{96}{50\cdot 8}\right)= \frac{19}{480\gamma^2 \epsilon}$ as required.

\end{document}